\documentclass[letterpaper, 10 pt, journal, twoside]{ieeetran}
\ifCLASSINFOpdf
\else
\fi
\usepackage{nccmath} 
\usepackage{mathtools} 
\usepackage{stackengine} 
\usepackage{amsmath,array} 
\usepackage{amsthm} 
\usepackage[noadjust]{cite} 
\usepackage{graphicx} 
\usepackage{caption,subcaption} 
\usepackage{float} 
\usepackage{epstopdf} 
\usepackage[capitalise]{cleveref} 
\usepackage{tikz} 
\usepackage{amsfonts} 
\usepackage[linesnumbered,ruled,vlined]{algorithm2e} 
\usepackage{ulem} 
\usepackage{kotex} 
\usepackage{textcmds} 

\makeatletter
\newcommand*{\rom}[1]{\expandafter\@slowromancap\romannumeral #1@}
\makeatother

\newtheorem{theorem}{Theorem}

\usepackage[noend]{algpseudocode}
\makeatletter
\def\BState{\State\hskip-\ALG@thistlm}
\makeatother            

\DeclareMathOperator*{\argmax}{arg\,max}

\DeclareSymbolFont{matha}{OML}{txmi}{m}{it}
\DeclareMathSymbol{\varv}{\mathord}{matha}{118}

\SetKwInput{KwInput}{Input}                
\SetKwInput{KwOutput}{Output}              
\SetKwProg{Fn}{Function}{:}{\KwRet}
\SetKw{Continue}{continue}
\SetKw{Not}{not}
\SetKw{And}{and}

\begin{document}
%
\title{
Fully Distributed Informative Planning for Environmental Learning with Multi-Robot Systems}
%
%
%

\author{Dohyun Jang$^{1}$, Jaehyun Yoo$^{2}$, Clark Youngdong Son$^{3}$, and H. Jin Kim$^{1}$%
\thanks{$^{1}$Dohyun Jang and H. Jin Kim are with the Department of Aerospace Engineering, Seoul National University (SNU), Seoul, South Korea (e-mail: dohyun@snu.ac.kr; hjinkim@snu.ac.kr).}%
\thanks{$^{2}$Jaehyun Yoo is with the School of AI Convergence, Sungshin Women's University, Seoul, South Korea (e-mail: jhyoo@sungshin.ac.kr).}%
\thanks{$^{3}$Clark Youngdong Son is with Mechatronics R\&D Center, Samsung Electronics, Hwaseong, South Korea (e-mail: y0dong.son@samsung.com).}%
\thanks{Digital Object Identifier (DOI): see top of this page.} } 

%
%

\markboth{IEEE} {IEEE}   

%



\maketitle


\begin{abstract}
This paper proposes a cooperative environmental learning algorithm working in a fully distributed manner. A multi-robot system is more effective for exploration tasks than a single robot, but it involves the following challenges: \romannumeral 1) online distributed learning of environmental map using multiple robots; \romannumeral 2) generation of safe and efficient exploration path based on the learned map; and \romannumeral 3) maintenance of the scalability with respect to the number of robots. To this end, we divide the entire process into two stages of environmental learning and path planning. Distributed algorithms are applied in each stage and combined through communication between adjacent robots. The environmental learning algorithm uses a distributed Gaussian process, and the path planning algorithm uses a distributed Monte Carlo tree search. As a result, we build a scalable system without the constraint on the number of robots. Simulation results demonstrate the performance and scalability of the proposed system. Moreover, a real-world-dataset-based simulation validates the utility of our algorithm in a more realistic scenario.

\end{abstract}


%
\IEEEpeerreviewmaketitle

\section{Introduction}
Robotic sensor networks, which combine the local sensing capabilities of various sensors with the mobility of robots, can provide more versatility than conventional fixed sensor networks due to their capability to extend the sensing range and improve the resolution of sensory data maps \cite{deng2017energy}. These networks have been studied extensively in survey of global environment \cite{matthew2017boustrophedon,feng2020an,patrikar2020wind,ma2016an}, industrial environment perception \cite{zhang2019intelligent}, radio signal search \cite{li2020multi}, and so on.

To construct sensor networks, we first deploy many sensors in a working space. Then, we establish communication channels with the central server to collect and fuse data acquired from all sensors. Since the wireless communication range of sensors is limited, sensors usually make an indirect connection with the central server, such as a mesh network that connects all sensors and the central server by relay channels.

However, the relay network requires a routing table that must be rebuilt every time the robot network is reconfigured, which is cumbersome for robotic sensor networks. This problem is particularly noticeable in unmanned aerial vehicles (UAVs) or small robots since they need to use relatively weak communication modules to reduce power consumption.

Decentralizing the system can be a proper solution to network problems by removing the dependency of robots on the central server. For example, if a robot can infer the entire sensory map only from the local information directly provided by surrounding robots, the search task can be completed without the help of the central server. This paper applies decentralization to the environmental learning phase and the path planning phase, respectively. With an online information fusion algorithm, we build a distributed autonomous system of multiple robots to search and learn even dynamic environments that change over time.

\begin{figure}[t]
\begin{center}
\includegraphics[width=0.45\textwidth]{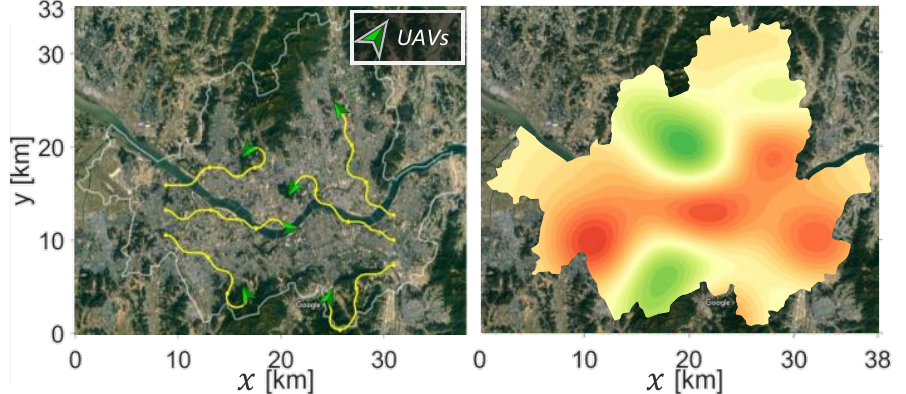}
\caption{Temperature monitoring simulation in Seoul using multiple UAVs; (left) trajectories of UAVs performing cooperative work through a distributed communication network, and (right) the reconstructed temperature map.} 
\label{fig:overview}
\end{center}
\vspace{-0.5cm}
\end{figure}

\subsection{Literature Review}
The first part of our work is multi-robot environmental learning in a distributed manner. For environmental learning, some useful techniques exist such as Gaussian mixture model (GMM) \cite{shi2020adaptive,luo2019distributed,gu2008distributed}, finite element method (FEM) \cite{elwin2020distributed}, and Gaussian process (GP) regression \cite{patrikar2020wind,ma2016an,li2020multi}. In particular, GP is a popular approach that derives a spatial relationship between sampled data using a kernel and performs Bayesian inference for prediction at an unknown region.

However, most GP-related researches focus on centralized systems, making it difficult to expand to large-scale multi-robot systems due to network resource limitations such as channel bandwidth and transmit power. Distributed multi-agent Gaussian regression is introduced in \cite{pillonetto2019distributed}, which designs a finite-dimensional GP estimator by using Karhunen–Loève (KL) expansion \cite{levy2008karhunen}. In contrast to the decentralized GP presented in \cite{viseras2016decentralized}, the distributed GP provides a common copy of the global estimate to all agents by exchanging the estimated information with their neighbors. This paper extends \cite{jang2020multi}, which shows that distributed GPs can construct environmental models using mobile robots in order to take the distributed path planning into account.

The second part of our work is informative path planning in a distributed multi-robot system. As an initial study of informative path planning, the problem of optimal sensor placement has been investigated to create an environmental map in a given space by properly placing a finite number of sensors \cite{krause2008near,luo2019distributed,gu2008distributed}.
Since then, by applying GPs and information theory, the research of optimal sensor placement has grown into the informative path planning research as presented in \cite{ma2016an,shi2020adaptive,flaspohler2019information,silveria2018reconstruction,ma2017informative,meliou2007nonmyopic,bottarelli2019orienteering}. Some studies have combined GP with conventional planning algorithms such as rapidly-exploring random tree (RRT) \cite{yang2013a}, dynamic programming (DP) \cite{ma2017informative}, or Monte Carlo tree search (MCTS) \cite{chen2019pareto}.

Besides the above approaches that mainly focus on informative path planning for single agents, many studies have applied informative planning for multi-robot systems. In \cite{viseras2016decentralized,li2020multi}, although both studies deal with decentralized multi-robot exploration using GP, these algorithms are not scalable as they consider only two robots. \cite{du2021parallelized} introduced the combination of the Kalman filter (KF) and the reduced value iteration (RVI) method for the parallelized active information gathering. While this technique is scalable to a large number of robots, it is noted that the environmental model has to be known, and only discrete environments can be represented since the model is expressed in KF. Considering the scalability for multi-robot systems, we extend the MCTS path planning in a distributed manner to be compatible with the distributed GP.

\subsection{Our Contribution}
To achieve our goal of fully distributed multi-robot informative planning, we divide the whole process into two phases: environmental learning and path planning. During these phases, we focus on three main contributions as follows.

• We develop an online distributed GP algorithm for environmental learning through Karhunen–Loève expansion and an infinite impulse response filter. This algorithm is capable of learning a dynamic environment.

• We propose a distributed informative path planning algorithm using a distributed MCTS combined with GP. In addition, we introduce the trajectory merging method to consider predicted trajectories of other agents.

• We build a fully distributed exploration and learning architecture using only local peer-to-peer communication for system scalability, as shown in Table \ref{table:scalability}.

We perform a multi-robot exploration simulation with a virtual environment setting and real-world dataset \cite{ncdc} provided by the National Climate Data Center (NCDC) in South Korea as shown in Fig. \ref{fig:overview}. 

The outline of this paper is as follows. Section \rom{2} briefly describes a multi-robot system setup and preliminaries. Section \rom{3} presents a method for online distributed environmental learning. Section \rom{4} combines environmental learning and MCTS in the distributed system. Simulations for the synthetic environment and real-world dataset are presented in Section \rom{5}. Section \rom{6} concludes the paper.

\begin{table}[ht]
\caption{Scalability comparison between centralized and distributed systems for multi-agent tasks. See text for symbols.}
\label{table:scalability}
\centering
\begin{tabular}{|l||l|l|}
\hline
                      & \multicolumn{1}{c|}{\textbf{Centralized}} & \multicolumn{1}{c|}{\textbf{Distributed (ours)}} \\ \hline
\begin{tabular}[c]{@{}l@{}}GP Computation\\ Complexity\end{tabular} & $O((mn)^3)$\quad\hfill (\cite{li2020multi,viseras2016decentralized})& $O(E^3)$\quad\hfill (\cite{jang2020multi})                                         \\ \hline
\begin{tabular}[c]{@{}l@{}}MCTS Planner \\Action Cardinality\end{tabular} & $\lvert \mathcal{A} \rvert^n$\quad\hfill (\cite{kartal2016monte})& $\lvert \mathcal{A} \rvert$\quad\hfill (\cite{best2019dec})\\ \hline
\begin{tabular}[c]{@{}l@{}}Communication\\ Complexity\end{tabular} & $O(n^2)$\quad\hfill (\cite{kartal2016monte,shi2020adaptive})& $O(n)$\quad\hfill (\cite{li2020multi,viseras2016decentralized,jang2020multi})                            \\ \hline
\end{tabular}
\vspace{-0.2cm}
\end{table}

\section{Multi-Robot System Setup and Preliminaries}\label{sec:problem_statement}
We focus on the environment learning problem in multi-robot systems by considering a target domain as a 3-dimensional compact set $\mathbb{X}_w\subset\mathbb{R}^3$. Multiple robots (e.g., ground vehicles or UAVs with onboard sensors) explore an unknown area and estimate environmental information using both self-measurements and shared data received from neighbors. All robots can discover obstacles nearby using the range sensor and only communicate with adjacent robots within the communication distance.

\begin{figure*}[ht]
\begin{center}
\includegraphics[trim = 0mm 0mm 0mm 0mm, clip, width=0.8\textwidth]{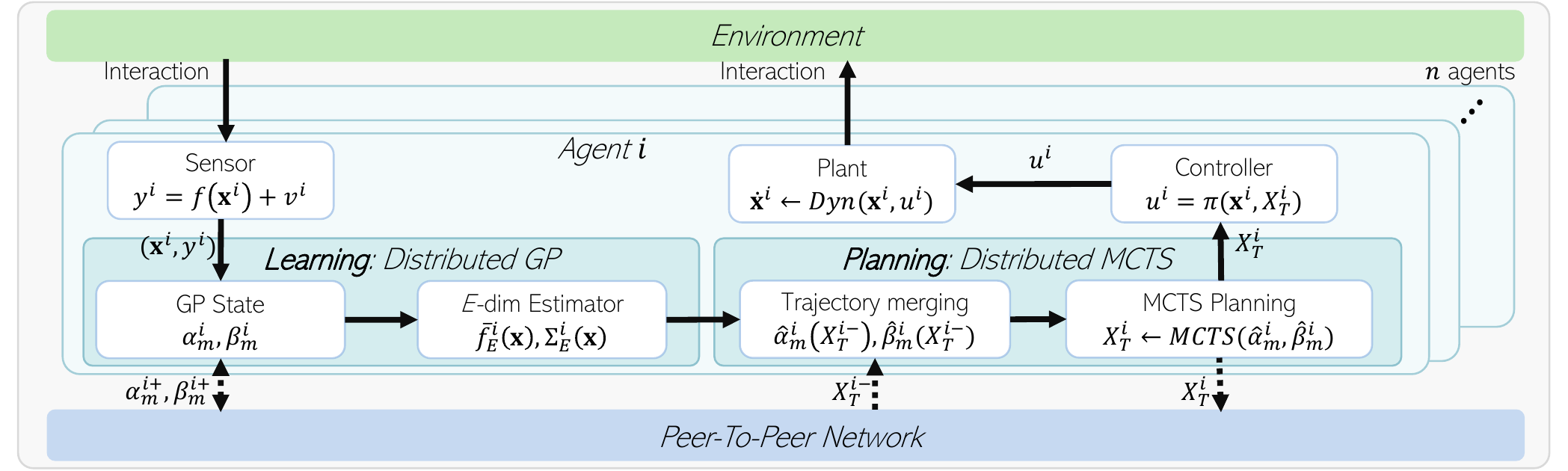} 
\caption{Structure of the distributed exploration and environmental model learning system. Each agent has its own distributed GP and distributed MCTS planner modules that operate through peer-to-peer communication with each other.} 
\label{fig:structure}
\end{center}
\vspace{-0.5cm}
\end{figure*}

\subsection{Multi-Robot System Setup}\label{subsec:multi-robot_sytem_setup}
As depicted in Fig. \ref{fig:overview}, we consider $n$ robot agents exploring the environment. Each robot $i$ takes the measurement $y_k^i$ of an unknown environmental process $f(\cdot)$ in its position $\mathbf{x}_{k}^{i}\in \mathbb{X}_w$ ($i=1,\cdots,n$) at time $k$ which has the following relationship:
\begin{ceqn}
\begin{equation}\label{eq:measurement_model}
\renewcommand{\arraystretch}{1.2}
\setlength{\arraycolsep}{2pt}%
\begin{array}{ r>{{}}l @{\quad} l @{\quad} r>{{}}l @{\quad} l }
y_{k}^{i}&=&f(\mathbf{x}_{k}^{i})+v_{k}^{i},\\
\end{array}
\end{equation} 
\end{ceqn}
where the measurement of $f(\mathbf{x}_{k}^{i})$ is corrupted by the additive white Gaussian noise $v_{k}^{i}\sim\mathcal{N}(0,\sigma_{v}^{2})$.

Each robot has its process modules, \textit{Distributed GP} and \textit{Distributed MCTS}, for the distributed monitoring task. During these processes, they share GP variables and predicted trajectories through a peer-to-peer communication network. This operation process is summarized in Fig. \ref{fig:structure}. The controller design process is not covered in this work.

To implement the communication network of $n$ robots, we define a set of neighbors for robot $i$ as $\mathcal{N}_{k}^{i}=\{j|\ {||\mathbf{x}_{k}^{i}-\mathbf{x}_{k}^{j}||}<d_{comm},j\in\mathcal{N}/i\}$, where $\mathcal{N}=\{1,2,\cdots,n\}$ is the index set of agents and $d_{comm}$ is the communication range. $\mathcal{N}^{i+}$ means $\mathcal{N}^{i}\cup \{i\}$. For arbitrary variable $A$, $A^{i+}$ means $\{A^{j}\}_{j\in \mathcal{N}^{i+}}$, and $A^{i-}$ means $\{A^{j}\}_{j\in \mathcal{N}^{i}}$ for brevity.

\subsection{Conventional Gaussian Process}\label{subsec:conventional_gp}
GP regression, which is data-driven non-parametric learning, can provide Bayesian inference over the set $\mathbb{X}_w$, taking into account joint Gaussian probability distribution between the sampled dataset \cite{rasmussen2003gaussian}. In \eqref{eq:measurement_model}, the unknown process model $f(\cdot)$ is assumed to follow a zero-mean Gaussian process as

\begin{ceqn}
\begin{equation}\label{eq:process_model}
\renewcommand{\arraystretch}{1.2}
\setlength{\arraycolsep}{2pt}%
\begin{array}{ r>{{}}l @{\quad} l @{\quad} r>{{}}l @{\quad} l }
f(\cdot)\sim \mathcal{G}\mathcal{P}(0,\kappa(\mathbf{x'},\mathbf{x''})).\\
\end{array}
\end{equation} 
\end{ceqn}
$\kappa(\mathbf{x'},\mathbf{x''})$ is a \textit{kernel} or \textit{covariance function} for positions $\mathbf{x'},\mathbf{x''}\in\mathbb{X}_{w}$, and the original \textit{squared exponential} (SE) kernel is defined as

\begin{ceqn}
\begin{align}\label{eq:kernel}
\kappa(\mathbf{x'},\mathbf{x''})=\sigma_{s}^2\text{exp}(-\frac{1}{2}(\mathbf{x'}-\mathbf{x''})^\top\Sigma_{l}^{-1}(\mathbf{x'}-\mathbf{x''})),
\end{align}
\end{ceqn}
where $\sigma_{s}^2$ is the signal variance of $f(\cdot)$, and $\Sigma_{l}$ is the length scale. The hyper parameters $\sigma_{s}^2$ and $\Sigma_{l}$ can be determined by maximizaing the marginal likelihood \cite{rasmussen2003gaussian}.

Formally, let $D_k^i=\{(\mathbf{x}_t^i,y_t^i)\}_{t\in M_k}$ be the training dataset sampled by the robot $i$, where $t$ is the sampling time. $M_k$ is the set of sampling time indices up to time $k$. With the dataset $D_k^i$ of size $m_k^i$, we can simply define the input data matrix as $\mathbf{X}_k^i=[\mathbf{\bar{x}}_{1}^i,\cdots,\mathbf{\bar{x}}_{m_k^i}^i]^\top\in \mathbb{R}^{m_k^i\times 3}$ and the output data vector as
$\mathbf{y}_k^i=[\bar{y}_{1}^i, \cdots, \bar{y}_{m_k^i}^i]^\top\in \mathbb{R}^{m_k^i\times 1}$. According to the test point $\mathbf{x}\in\mathbb{X}_w$, the posterior distribution over $f(\mathbf{x})$ by robot $i$ is derived as follows:
\begin{ceqn}
\begin{equation}\label{eq:posterior_distribution}
p(f(\mathbf{x})|\mathbf{X}_k^i,\mathbf{y}_k^i,\mathbf{x}) \sim \mathcal{N}(\bar{f}^i(\mathbf{x}),\Sigma^i(\mathbf{x})),\\
\end{equation} 
\end{ceqn}
where
\begin{subequations}\label{eq:GP_mean_var}
\begin{align}
\bar{f}^i(\mathbf{x})&=\mathbf{k}^\top(\mathbf{X}_k^i,\mathbf{x})(\mathbf{K}(\mathbf{X}_k^i,\mathbf{X}_k^i)+\sigma_{v}^{2}I)^{-1}\mathbf{y}_k^i \label{eq:GP_mean_var(a)},\\
\begin{split}
    \Sigma^i(\mathbf{x})&= \kappa(\mathbf{x},\mathbf{x}) \label{eq:GP_mean_var(b)}\\
    &\ \ -\mathbf{k}^\top(\mathbf{X}_k^i,\mathbf{x})(\mathbf{K}(\mathbf{X}_k^i,\mathbf{X}_k^i)+\sigma_{v}^{2}I)^{-1}\mathbf{k}(\mathbf{X}_k^i,\mathbf{x}).
  \end{split}
\end{align}
\end{subequations}
$\mathbf{K}(\mathbf{X}_k^i,\mathbf{X}_k^i)$ is the $m_k^i\times m_k^i$ kernel matrix whose $(u,v)$-th element is $\kappa(\mathbf{\bar{x}}_u^i,\mathbf{\bar{x}}_v^i)$ for $\mathbf{\bar{x}}_u^i,\mathbf{\bar{x}}_v^i \in \mathbf{X}_k^i$. $\mathbf{k}(\mathbf{X}_k^i,\mathbf{x})$ is the $m_k^i\times 1$ column vector that is also obtained in the same way.

\subsection{Informative Path Planning}\label{subsec:informative_planning}

To obtain the better description of a spatial process model, robots perform informative path planning. It maximizes the \textit{information gain} $\mathbb{I}(;)$, which is the mutual information between the process $f$ and measurements $\mathbf{y}$:

\begin{align}\label{eq:information_gain}
\mathbb{I}(\mathbf{y};f)=\textrm{H}(\mathbf{y})-\textrm{H}(\mathbf{y}|f),
\end{align}
where $\textrm{H}(\cdot)$ is the \textit{entropy} of a random variable. Let $X_{1:k}^i$ be the possible trajectory of robot $i$ and $X_{1:k}=\{X_{1:k}^1,\cdots,X_{1:k}^n\}$ be the possible trajectories of all robots. Then, the multi-robot team's global objective function $J(X_{1:k})$ is defined as follows:

\begin{align}\label{eq:global_objective_function}
J(X_{1:k})=\mathbb{I}(\mathbf{y}_{1:k};f).
\end{align}
$\mathbf{y}_{1:k}$ is the measurements corresponding to $X_{1:k}$. As a result, the optimal trajectories for all agents are defined as follows:

\begin{ceqn}
\begin{equation}\label{eq:optimal_trajectory_all}
\renewcommand{\arraystretch}{1.2}
\setlength{\arraycolsep}{-1pt}%
\begin{array}{ r>{{}}l @{\ } l @{\ } r>{{}}l @{\ } l }
X_{1:k}^{*}&=&\argmax\limits_{X_{1:k}}J(X_{1:k})\\
&=&\argmax\limits_{X_{1:k}}\mathbb{I}(\mathbf{y}_{1:k};f)\\
&=&\argmax\limits_{X_{1:k}}(\textrm{H}(\mathbf{y}_{1:k})-\textrm{H}(\mathbf{y}_{1:k}|f)).
\end{array}
\end{equation} 
\end{ceqn}
In this study, the entropy $\textrm{H}(\cdot)$ is obtained using GP. With the result of GP estimation \eqref{eq:GP_mean_var}, the optimal trajectory generation for each agent will be addressed in Section \ref{sec:monte_carlo}.

\section{Environmental Learning: Distributed Gaussian Process}\label{sec:spatial_model_learning}
In this section, we expand the conventional GP in Section \ref{subsec:conventional_gp} to the distributed GP. The first step is to expand the conventional kernel \eqref{eq:kernel} to be an infinite sum of eigenfunctions. Then, the expanded kernel is used to make a finite-dimensional GP estimator, and the estimator is reformulated to a distributed form. With a consecutive state update rule, the GP estimator works in a distributed manner.

\subsection{Karhunen–Loève (KL) Kernel Expansion}
Let the usual GP consider $n$ robots. We can simply define the input data matrix for $n$ robots as $\mathbf{X}_k=[\mathbf{X}_k^{1\top},\cdots,\mathbf{X}_k^{n\top}]^\top\in \mathbb{R}^{mn\times 3}$. For simplicity, it is assumed that $m_k^i$'s are same for all robots, and we omit the subscript $k$, so $m^i_k=m$ hereafter. With the matrix $\mathbf{X}_k$, the usual GP requires all the sampled data $\mathbf{X}_k$ and inversion of $K(\mathbf{X}_k,\mathbf{X}_k)$ with $O((mn)^{3})$ operations. These requirements are impractical when peer-to-peer communication is only used, and the computational burden also increases depending on the data size. For this reason, a new kernel method is needed. The kernel \eqref{eq:kernel} can be expanded in terms of eigenfunctions $\phi_{e}$ and corresponding eigenvalues $\lambda_{e}$ as follows \cite{levy2008karhunen}:

\begin{align}\label{eq:eigenfunction_decomposition2}
\kappa(\mathbf{x}',\mathbf{x}'')=\sum\limits_{e=1}^{+\infty}\lambda_{e}\phi_{e}(\mathbf{x}')\phi_{e}(\mathbf{x}''),
\end{align}
where $\lambda_{e}\phi_{e}(\mathbf{x}')=\int_{\mathbb{X}_w}\kappa(\mathbf{x}',\mathbf{x}'')\phi_{e}(\mathbf{x}'')d\mu(\mathbf{x}'')$.
It is difficult to derive the kernel eigenfunctions in a closed-form, but the SE kernel expansion has already been obtained via Hermite polynomials, as mentioned in \cite{zhu1997gaussian}. Then, the process model $f$ for the position $\mathbf{x}\in\mathbb{X}_w$ is expanded as

\begin{align}\label{eq:measurement_model2}
\begin{array}{r @{\ } l @{\ } l>{{}}l @{\ } l }
f(\mathbf{x})&=&\sum\limits_{e=1}^{E}a_{e}\phi_{e}(\mathbf{x})+\sum\limits_{e=E+1}^{+\infty}a_{e}\phi_{e}(\mathbf{x})\\
&=&f_{E}(\mathbf{x})+\sum\limits_{e=E+1}^{+\infty}a_{e}\phi_{e}(\mathbf{x}),
\end{array}
\end{align}
where $a_{e}\sim \mathcal{N}(0,\lambda_{e})$ for $e=1,2,\cdots,\infty$. 
$f_{E}(\mathbf{x})$ is the $E$-dimensional model of $f(\mathbf{x})$ where $E$ is a constant design parameter. This parameter can be tuned by the SURE strategies \cite{pillonetto2019distributed}. As shown in \cite{zhu1997gaussian}, the optimal $E$-dimensional models can be obtained by a convex combination of the first $E$-kernel eigenfunctions as the size of sampled dataset increases to infinity.

\subsection{Multi-Agent Distributed Gaussian Process}\label{subsec:Multi-agent distributed Gaussian process}
We apply $E$-dimensional approximation to the GP estimator in \eqref{eq:GP_mean_var} to derive the estimation of $f_E(\mathbf{x})$. According to $E$-dimensional approximation, the kernel function \eqref{eq:eigenfunction_decomposition2} can be described as $\kappa(\mathbf{x}',\mathbf{x}'')\approx \sum\limits_{e=1}^{E}\lambda_{e}\phi_{e}(\mathbf{x}')\phi_{e}(\mathbf{x}'')$.
For the input data matrix $\mathbf{X}_k$, kernel matrices included in \eqref{eq:GP_mean_var} are defined by

\begin{subequations}
\begin{align}\label{eq:kernel2}
\renewcommand{\arraystretch}{1.3}
\mathbf{K}(\mathbf{X}_k,\mathbf{X}_k)&= G\Lambda_{E}G^\top,\\
\mathbf{k}(\mathbf{X}_k,\mathbf{x})&=G\Lambda_{E}\Phi(\mathbf{x}),
\end{align}
\end{subequations}
where $\Phi(\mathbf{x}):=
  \begin{bmatrix}
    \phi_{1}(\mathbf{x}),\cdots,\phi_{E}(\mathbf{x})
  \end{bmatrix}^\top$ and $G:=
  \begin{bmatrix}
    \Phi(\mathbf{\bar{x}}_{1}^{1})\cdots\Phi(\mathbf{\bar{x}}_{m}^{1}),\cdots,\Phi(\mathbf{\bar{x}}_{1}^{n})\cdots\Phi(\mathbf{\bar{x}}_{m}^{n})
  \end{bmatrix}^\top$. $\Lambda_{E}$ is the diagonal matrix of kernel eigenvalues.

With \eqref{eq:GP_mean_var(a)} and \eqref{eq:kernel2}, the $E$-dimensional estimator for GP is expressed as follows \cite{pillonetto2019distributed}:

\begin{align}\label{eq:GP_E_dim_mean_estimator}
\renewcommand{\arraystretch}{1.2}
\setlength{\arraycolsep}{2pt}%
\begin{array}{ r>{{}}l @{\quad} l @{\quad} r>{{}}l @{\quad} l }
\bar{f}_{E}(\mathbf{x})&:=&
  \Phi^\top(\mathbf{x})H_{E}\mathbf{y},
\end{array}
\end{align}
where

\label{eq:GP_E_dim_mean_estimator_sub}
\begin{align}
H_{E}&:=
  \left(\dfrac{G^\top G}{mn}+\dfrac{\sigma_{v}^{2}}{mn}\Lambda_{E}^{-1}\right)^{-1}\dfrac{G^\top}{mn}.
\label{eq:GP_E_dim_mean_estimator_sub(b)}
\end{align}

Because each agent cannot obtain $G$ and $\mathbf{y}$ in \eqref{eq:GP_E_dim_mean_estimator} without a fully connected network, we decompose the associated terms included in \eqref{eq:GP_E_dim_mean_estimator_sub(b)} as follows:
\begin{subequations}\label{eq:GP_decomposition}
\begin{equation}\label{eq:GP_decomposition(a)}
\renewcommand{\arraystretch}{1.2}
\setlength{\arraycolsep}{2pt}%
\begin{array}{ r>{{}}l @{\quad} l @{\quad} r>{{}}l @{\quad} l }
\dfrac{G^\top G}{mn}&=
\dfrac{1}{mn}\sum\limits_{i=1}^{n}\sum\limits_{t=1}^{m}\Phi(\mathbf{\bar{x}}_{t}^{i})\Phi^\top(\mathbf{\bar{x}}_{t}^{i})=
\dfrac{1}{n}\sum\limits_{i=1}^{n}\alpha_{m}^{i},\\
\end{array}
\end{equation}
\begin{equation}\label{eq:GP_decomposition(b)}
\renewcommand{\arraystretch}{1.2}
\setlength{\arraycolsep}{2pt}%
\begin{array}{ r>{{}}l @{\quad\quad\quad\;\,} l @{\quad} r>{{}}l @{\quad} l }
\dfrac{G^\top\mathbf{y}}{mn}&=\dfrac{1}{mn}\sum\limits_{i=1}^{n}\sum\limits_{t=1}^{m}\Phi(\mathbf{\bar{x}}_{t}^{i})\bar{y}_{t}^{i}=
\dfrac{1}{n}\sum\limits_{i=1}^{n}\beta_{m}^{i},
\end{array}
\end{equation} 
\end{subequations}
where $\alpha_{m}^{i}:=\sum_{t=1}^{m}\Phi(\mathbf{\bar{x}}_{t}^{i})\Phi^\top(\mathbf{\bar{x}}_{t}^{i})/m$ and $\beta_{m}^{i}:=\sum_{t=1}^{m}\Phi(\mathbf{x}_{t}^{i})\bar{y}_{t}^{i}/m$ are \textit{GP states} after the $m$-th sensor measurements. Now \eqref{eq:GP_E_dim_mean_estimator} is reformulated in the following distributed form:

\begin{ceqn}
\begin{equation}\label{eq:GP_E_dim_mean_estimator2}
\renewcommand{\arraystretch}{1.2}
\setlength{\arraycolsep}{2pt}%
\begin{array}{ r>{{}}l @{\quad} l @{\quad} r>{{}}l @{\quad} l }
\bar{f}_{E}^{i}(\mathbf{x})&:=&
  \Phi^\top(\mathbf{x})\left(\alpha_{m}^{i}+\dfrac{\sigma_{v}^{2}}{mn}\Lambda_{E}^{-1}\right)^{-1}\beta_{m}^{i}.\\

\end{array}
\end{equation} 
\end{ceqn}
As the results of average consensus protocol \cite{saber2003consensus}, \eqref{eq:GP_E_dim_mean_estimator2} converges to \eqref{eq:GP_E_dim_mean_estimator} after iterative communication. Similarly, the distributed form of $\Sigma(\mathbf{x})$ in \eqref{eq:GP_mean_var(b)} is expressed as

\begin{ceqn}
\begin{equation}\label{eq:GP_E_dim_var_estimator}
\renewcommand{\arraystretch}{1.2}
\setlength{\arraycolsep}{2pt}%
\begin{array}{ r>{{}}l @{\quad} l @{\quad} r>{{}}l @{\quad} l }
\Sigma_{E}(\mathbf{x})
&:=&
\kappa(\mathbf{x},\mathbf{x})-\Phi^\top(\mathbf{x})H_{E}G\Lambda_{E}\Phi(\mathbf{x}),
\end{array}
\end{equation} 
\end{ceqn}

\begin{ceqn}
\begin{equation}\label{eq:GP_E_dim_var_estimator2}
\renewcommand{\arraystretch}{1.2}
\setlength{\arraycolsep}{-1pt}%
\begin{array}{ r>{{}}l @{\ } l @{\ } r>{{}}l @{\ } l }
\Sigma_{E}^{i}(\mathbf{x}):&=&\kappa(\mathbf{x},\mathbf{x})-\Phi^\top(\mathbf{x})
\left(\alpha_{m}^{i}+\dfrac{\sigma_{v}^{2}}{mn}\Lambda_{E}^{-1}\right)^{-1}\\
&&\times \alpha_{m}^{i} \Lambda_{E}
\Phi(\mathbf{x}).
\end{array}
\end{equation} 
\end{ceqn}

When we compare \eqref{eq:GP_mean_var} with \eqref{eq:GP_E_dim_mean_estimator2} and \eqref{eq:GP_E_dim_var_estimator2}, the computational complexity of the distributed algorithm is $O(E^{3})$, whereas that of the centralized algorithm is $O((mn)^{3})$ due to the matrix inversion \cite{pillonetto2019distributed}. Therefore, the distributed algorithm is more scalable since $E\ll mn$ in general.

\subsection{Online Information Fusion by Moving Agents}\label{subsec:online_information_fusion_by_moving_agents}
If the $(m+1)$-th new training dataset $\{(\mathbf{\bar{x}}_{m+1}^{i},\bar{y}_{m+1}^{i})\}_{i=1}^{n}$ are obtained, $\{\alpha_{m}^{i}\}_{i=1}^{n}$ and $\{\beta_{m}^{i}\}_{i=1}^{n}$ have to be discarded to include new data so that the consensus process must be restarted from scratch. To avoid repeated restarts and keep the continuity of environmental estimate, we introduce the online information fusion algorithm.

Let us assume that the sensor measurement frequencies of all agents are same for convenience. The update rule of $\alpha_{m}^{i}$ and $\beta_{m}^{i}$ is defined as follows:

\begin{ceqn}
\begin{equation}\label{eq:GP_alpha_beta}
\renewcommand{\arraystretch}{1.6}
\setlength{\arraycolsep}{2pt}%
\begin{array}{ r>{{}}l @{\ } l @{\ } r>{{}}l @{\ } l }

\alpha_{m+1}^{i}&=&(1-r)\alpha_{m}^{i}+r\Phi(\mathbf{\bar{x}}_{m+1}^{i})\Phi^\top(\mathbf{\bar{x}}_{m+1}^{i}),\\

 \beta_{m+1}^{i}&=&(1-r)\beta_{m}^{i}+r\Phi(\mathbf{\bar{x}}_{m+1}^{i})y_{m+1}^{i},
\end{array}
\end{equation} 
\end{ceqn}
where $\alpha_{0}^{i}=0$ and $\beta_{0}^{i}=0$. This rule is an infinite impulse response (IIR) filter. If $r=(m+1)^{-1}$, The update rule reflects all dataset equally, so it is suitable for static environmental learning. If $r>(m+1)^{-1}$, this rule reflects more of the recent data, so it is suitable for dynamic environmental learning. Simulation results for each environmental learning are shown in Chapter \ref{sec:simulation}.

\begin{theorem}\label{th:update}
Using an average consensus protocol and update rules in \eqref{eq:GP_alpha_beta}, new data are successively fused with the existing $\{\alpha_{m}^{i}\}_{i=1}^{n}$ and $\{\beta_{m}^{i}\}_{i=1}^{n}$, so that $\{\alpha_{m+1}^{i}\}_{i=1}^{n}$ and $\{\beta_{m+1}^{i}\}_{i=1}^{n}$ converge towards \eqref{eq:GP_decomposition(a)} and \eqref{eq:GP_decomposition(b)}, respectively, in a distributed manner.
\end{theorem}

\begin{proof}
See the Appendix in \cite{jang2020multi}.
\end{proof}

\begin{figure}
\begin{center}
\includegraphics[trim = 0mm 0mm 0mm 0mm, clip, width=8.5cm]{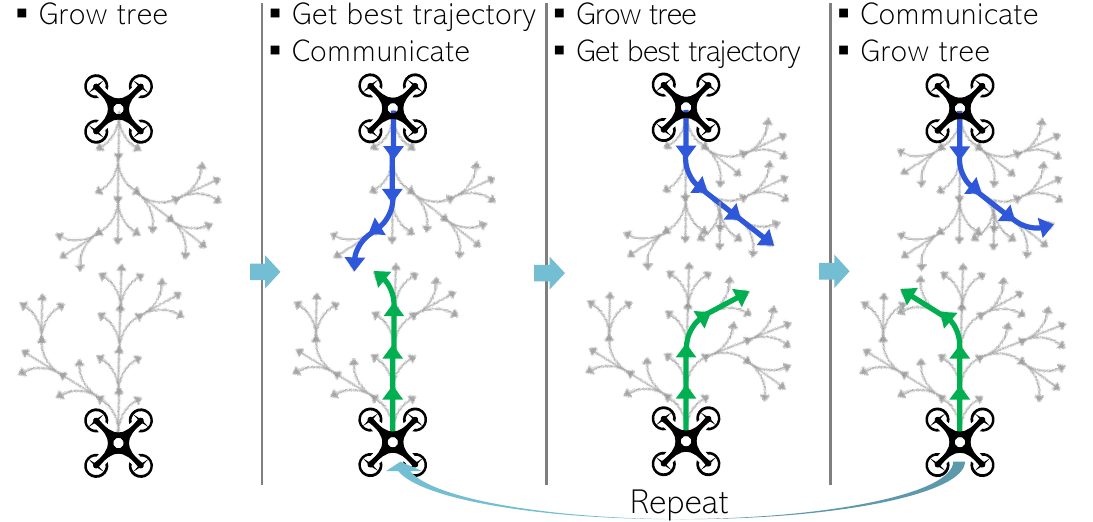} 
\caption{Illustration of the distributed MCTS process. Nearby robots exchange their predicted trajectories. These trajectories are used to temporarily update GPs and grow search trees. This process is repeated until the time budget is met.} 
\label{fig:distributed_mcts}
\end{center}
\vspace{-0.5cm}
\end{figure}

\section{Path Planning: Distributed Monte Carlo Tree Search}\label{sec:monte_carlo}
Using the distributed model learning discussed in the previous section, all agents create a local environmental map that converges to the global environmental map even in a distributed network. To find the most promising search trajectories with the learned map, all agents should consider every possible action. However, because the cardinality of possible action set grows exponentially with respect to the number of robots, the distributed planning strategy is needed in multi-robot path planning \cite{du2021parallelized}. In \cite{best2019dec}, the decentralized MCTS approach is studied to alleviate the cardinality of possible action set from $\lvert \mathcal{A} \rvert^n$ to $\lvert \mathcal{A} \rvert$, where $\mathcal{A}$ represents the discrete action space of each robot. We apply this advantage to our GP-based informative planning of multiple robots. With the distributed MCTS, each robot calculates a promising trajectory by communication with neighboring agents only. This process is shown in Fig. \ref{fig:distributed_mcts}. This section introduces the trajectory merging method to reflect the neighbor's path in each agent's tree search process. The contents of this section are summarized in Algorithm \ref{al:tree} and \ref{al:distributedMCTS}.

\subsection{Trajectory merging}
For each agent $i$, $X_{k+1:k+T}^i=(\mathbf{\hat{x}}_{k+1}^i,\cdots,\mathbf{\hat{x}}_{k+T}^i)$ denotes the predicted trajectory with the prediction length $T$, or it can be represented by $X_T^i$ for brevity. 
Assuming that the agent $i$ receives the predicted trajectories of neighboring agents $X_T^{i-}=\{{X_T^{j}}\}_{j\in\mathcal{N}_{i}}$, we modify the GP sate $\alpha_{m}^{i}$ as follows:
\begin{ceqn}
\begin{equation}\label{eq:GP_alpha_beta2}
\renewcommand{\arraystretch}{1.6}
\setlength{\arraycolsep}{2pt}%
\begin{array}{ r>{{}}l @{\ } l @{\ } r>{{}}l @{\ } l }

\hat{\alpha}_{m}^{i}&(X_T^{i-})=\dfrac{mn}{mn+\mathrm{n}(X_T^{i-})}\alpha_{m}^{i}\\
+&\dfrac{\mathrm{n}(X_T^{i-})}{mn+\mathrm{n}(X_T^{i-})}\sum\limits_{j\in\mathcal{N}_i}\sum\limits_{t=1}^{T}\Phi(\mathbf{\hat{x}}_{k+t}^{j})\Phi^\top(\mathbf{\hat{x}}_{k+t}^{j}).
\end{array}
\end{equation} 
\end{ceqn}
$\mathrm{n}(X_T^{i-})$ is the number of sensing points included in $X_T^{i-}$. With \eqref{eq:GP_alpha_beta2} and the $E$-dimensional estimator in \eqref{eq:GP_E_dim_var_estimator2}, we define the trajectory-merged GP estimator as follows:

\begin{ceqn}
\begin{equation}\label{eq:GP_E_dim_var_estimator3}
\renewcommand{\arraystretch}{1.2}
\setlength{\arraycolsep}{-1pt}%
\begin{array}{ r>{{}}l @{\ } l @{\ } r>{{}}l @{\ } l }
\hat{\Sigma}_{E}^{i}(\mathbf{x}):&=&\kappa(\mathbf{x},\mathbf{x})\\
&&-\Phi^\top(\mathbf{x})
\left(\hat{\alpha}_{m}^{i}+\dfrac{\sigma_{v}^{2}}{mn+\mathrm{n}(X_T^{i-})}\Lambda_{E}^{-1}\right)^{-1}\\
&&\times\hat{\alpha}_{m}^{i}\Lambda_{E}
\Phi(\mathbf{x}).
\end{array}
\end{equation} 
\end{ceqn}
In this way, predicted trajectories of neighboring agents are temporarily included in the acquired data set of the GP estimator. Because \eqref{eq:GP_alpha_beta2} and \eqref{eq:GP_E_dim_var_estimator3} are temporary values for tree search in distributed MCTS, they do not affect $\alpha_{m}^{i}$ and disappear after getting new predicted trajectories. This process is summarized in the Distributed MCTS block of Fig. \ref{fig:structure}. With this result, the path planning process will be explained in the next section.

\subsection{Informational reward function}\label{subsec:informational_reward}
As we mentioned in \eqref{eq:global_objective_function}, the information gain $\mathbb{I}$ is the objective function we have to maximize. With the definition in \eqref{eq:optimal_trajectory_all}, the optimal trajectory considering neighboring paths is defined as follows.

\begin{ceqn}
\begin{equation}\label{eq:optimal_trajectory}
\renewcommand{\arraystretch}{1.2}
\setlength{\arraycolsep}{-1pt}%
\begin{array}{ r>{{}}l @{\ } l @{\ } r>{{}}l @{\ } l }
X_T^{i*}&=&\argmax\limits_{X_T^{i}\in \mathbb{X}_{k}^{i}}J(X_T^{i}\cup X_T^{i-}\cup X_{1:k})\\
&=&\argmax\limits_{X_T^{i}\in \mathbb{X}_{k}^{i}}\mathbb{I}(\mathbf{y}_{T}^{i} \cup \mathbf{y}_{T}^{i-}\cup \mathbf{y}_{1:k};f),
\end{array}
\end{equation} 
\end{ceqn}
$\mathbf{y}_T^i$ and $\mathbf{y}_T^{i-}$ are the measurements corresponding to $X_T^i$ and $X_T^{i-}$, respectively. $\mathbb{X}_{k}^{i}$ is the domain of possible trajectories for agent $i$. As shown in \eqref{eq:information_gain}, information gain is represented with entropies as follows:

\begin{ceqn}
\begin{equation}
\renewcommand{\arraystretch}{1.2}
\setlength{\arraycolsep}{-1pt}%
\begin{array}{ r>{{}}l @{\ } l @{\ } r>{{}}l @{\ } l }
\mathbb{I}(\mathbf{y}_{T}^{i}\cup \mathbf{y}_{T}^{i-}\cup \mathbf{y}_{1:k};f)&=&\textrm{H}(\mathbf{y}_{T}^{i}\cup \mathbf{y}_{T}^{i-}\cup \mathbf{y}_{1:k})\\
&&-\textrm{H}(\mathbf{y}_{T}^{i}\cup \mathbf{y}_{T}^{i-}\cup \mathbf{y}_{1:k}|f).\\
\end{array}
\end{equation} 
\end{ceqn}
Using the mesuarement model \eqref{eq:measurement_model} and the entropy calculation for the normal distribution \cite{srinivas2012information}, conditional entropy becomes 

\begin{ceqn}
\begin{equation}
\renewcommand{\arraystretch}{1.2}
\setlength{\arraycolsep}{-1pt}%
\begin{array}{ r>{{}}l @{\ } l @{\ } r>{{}}l @{\ } l }
\textrm{H}(\mathbf{y}_{T}^{i}\cup \mathbf{y}_{T}^{i-}\cup \mathbf{y}_{1:k}|f)=\frac{1}{2}\log{|2\pi e \sigma_{v}^{2}I|}.
\end{array}
\end{equation} 
\end{ceqn}
$\textrm{H}(\mathbf{y}_{T}^{i}\cup \mathbf{y}_{T}^{i-}\cup \mathbf{y}_{1:k})$ is decomposed using conditional entropy, and it is obtained by calculating the entropy of GP as follows:

\begin{ceqn}
\begin{equation}
\renewcommand{\arraystretch}{1.2}
\setlength{\arraycolsep}{-1pt}%
\begin{array}{ r>{{}}l @{\ } l @{\ } r>{{}}l @{\ } l }
\textrm{H}(\mathbf{y}_{T}^{i}\cup \mathbf{y}_{T}^{i-}\cup \mathbf{y}_{1:k})&=&\textrm{H}(\mathbf{y}_{T}^{i-}\cup \mathbf{y}_{1:k})+\textrm{H}(\mathbf{y}_{T}^{i}|\mathbf{y}_{T}^{i-}\cup \mathbf{y}_{1:k})\\
&\approx&\textrm{H}(\mathbf{y}_{T}^{i-}\cup \mathbf{y}_{1:k})+\frac{1}{2}\log{|2\pi e \hat{\Sigma}_{E}^{i}(X_{T}^{i})|}.\\
\end{array}
\end{equation}
\end{ceqn}
As a result, with the trajectory-merged GP estimator in \eqref{eq:GP_E_dim_var_estimator3}, the optimal trajectory for agent $i$ is defined as follows:

\begin{ceqn}
\begin{equation}\label{eq:optimal_trajectory2}
\renewcommand{\arraystretch}{1.2}
\setlength{\arraycolsep}{-1pt}%
\begin{array}{ r>{{}}l @{\ } l @{\ } r>{{}}l @{\ } l }
X_{T}^{i*}&=&\argmax\limits_{X_{T}^{i}\in \mathbb{X}_{k}^{i}}J(X_{T}^{i}\cup X_{T}^{i-}\cup X_{1:k})\\
&\approx&\argmax\limits_{X_{T}^{i}\in \mathbb{X}_{k}^{i}}\frac{1}{2}\log{|2\pi e \hat{\Sigma}_{E}^{i}(X_{T}^{i})|}\\
&=&\argmax\limits_{X_{T}^{i}\in \mathbb{X}_{k}^{i}}\mathcal{R}^{i}(X_{T}^{i}).\\
\end{array}
\end{equation} 
\end{ceqn}
We call $\mathcal{R}^{i}(\cdot)$ the \textit{informational reward function}, which is utilized in the tree search algorithm.

\begin{figure}
\begin{center}
\includegraphics[width=0.45\textwidth]{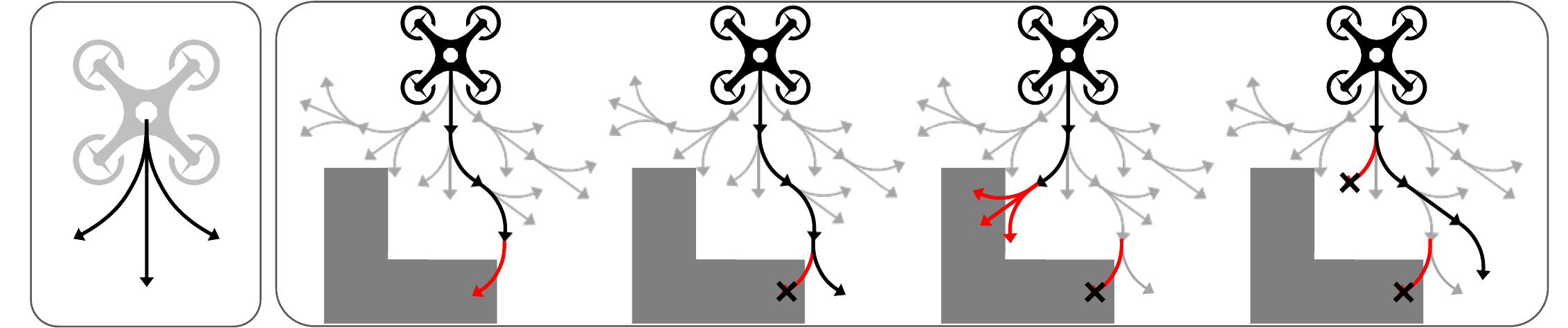} 
\caption{Node closing method for obstacle avoidance in MCTS. (left) finite action space that the robot can take. (right) tree expansion and node closing process.} 
\label{fig:closing}
\end{center}
\vspace{-0.5cm}
\end{figure}

\subsection{Tree Search with D-UCB Alogrithm}
Using the informational reward function defined in \ref{subsec:informational_reward}, the tree search algorithm iteratively explores and evaluates predictive path candidates according to the discounted upper confidence bound (D-UCB) rule to find the optimal path. D-UCB rule assigns the probabilistic search priority to the action candidates.

The tree structure consists of nodes $s$ and edges $(s,a)$ for all legal actions $a\in \mathcal{A}(s)$. Each edge contains a set of variables $\{N_{s}^{a}, W_{s}^{a}, \tau_{s}^{a}, C_{s}^{a}\}$ where $N_{s}^{a}$ is the visit count, $W_{s}^{a}$ is the total action value, $\tau_{s}^{a}$ is the number of iterations for tree search (shown in line 5 of Algorithm \ref{al:distributedMCTS}), and $C_{s}^{a}$ is a closing variable which will be discussed. We follow the tree search process in Algorithm \ref{al:tree} and the distributed MCTS with GP in Algorithm \ref{al:distributedMCTS}. The MCTS process can be divided into four main steps as follows.

\subsubsection{Selection} (lines 4, 12-22 of Algorithm \ref{al:tree}) The selection phase focuses on finding a leaf node $s_{\textit{leaf}}$. Following the D-UCB rule, the selected action at node $s$ is defined as follows \cite{srinivas2012information}:

\begin{ceqn}
\begin{equation}\label{eq:UCB_action}
\renewcommand{\arraystretch}{1.2}
\setlength{\arraycolsep}{-1pt}%
\begin{array}{ r>{{}}l @{\ } l @{\ } r>{{}}l @{\ } l }
a_{t}=\argmax\limits_{a\in\mathcal{A}(s)}\left(\dfrac{W_{s}^{a}\gamma^{\tau-\tau_{s}^{a}}}{N_{s}^{a}\gamma^{\tau-\tau_{s}^{a}}}+U_{s}^{a}\right),
\end{array}
\end{equation} 
\end{ceqn}
where
\begin{ceqn}
\begin{equation}\label{eq:UCB1}
\renewcommand{\arraystretch}{1.2}
\setlength{\arraycolsep}{-1pt}%
\begin{array}{ r>{{}}l @{\ } l @{\ } r>{{}}l @{\ } l }
U_{s}^{a}=\sqrt{\dfrac{\ln\sum_{a'\in\mathcal{A}(s)}{N_{s}^{a'}\gamma^{\tau-\tau_{s}^{a'}}}}{N_{s}^{a}\gamma^{\tau-\tau_{s}^{a}}}}.
\end{array}
\end{equation} 
\end{ceqn}
The first term on the right-hand side of \eqref{eq:UCB_action} means exploration term for the tree search, and the second term means exploitation term. As shown in Algorithm \ref{al:distributedMCTS}, each agent periodically receives the predicted trajectories of adjacent agents, which are utilized in the tree search process. It means that the tree, obtained by using previously given trajectories, may not be optimal when new neighboring trajectories are received. Therefore, adopting the discount factor $\gamma$ makes the previously visited nodes less influential on the current UCB value.

If the current node has no selectable actions because of path blockage, the node closes ($C_{s}^{a}\leftarrow 1$) and the algorithm returns to the parent node to restart the selection process. We call this process as node closing method illustrated in Fig. \ref{fig:closing} and lines 15-21 of Algorithm \ref{al:tree}.

\renewcommand{\algorithmiccomment}[1]{\(\hfill\hfill\hfill\hfill\hfill\hfill\hfill\hfill\hfill\hfill\hfill\hfill\hfill\hfill\hfill\hfill\triangleright\ \)#1}
\normalem 
\begin{algorithm}[t]
\caption{Tree Search with Obstacle Avoidance}\label{al:tree}
\DontPrintSemicolon
	\Fn{TreeSearch$(T_{k}^{i},\tau)$}{
  	$s_{0}\leftarrow$ \emph{getRootNode}$(T_{k}^{i})$\;
  	\For{$t\leftarrow 1\ ${\normalfont to }$N_{\text{iteration}}$}    
    {
  		$s_{\textit{leaf}}\leftarrow$ \emph{selection}$(s_{0},\tau)$\;
  		$(s_{t},a)\leftarrow$ \emph{expansion}$(s_{\textit{leaf}},\tau)$\;
  		\If{collisionCheck$(s_{t})$}{
  		$C_{s_{\textit{leaf}}}^{a}\leftarrow$ \textit{true}\;   
  		\Continue\;
  		}
  		$r_{t}\leftarrow$ \emph{simulation}$(s_{t})$\;
  		\emph{backprop}$(s_{t},r_{t})$\;
		}
		\KwRet $T_{k}^{i}$\;
  }

	\Fn{selection$(s_{\textit{leaf}},\tau)$}{
		\While{\Not leafNodeFound}
		{
			\If{\Not $depth(s_{leaf})>T$}
			{
				\If{$C_{s_{\textit{leaf}}}^{a}=\textit{true}\ \forall{a\in\mathcal{A}(s_{\textit{leaf}})}$}
				{
					$C_{s_{\textit{leaf}-1}}^{a_{\textit{leaf}-1}}\leftarrow \textit{true}$\;
					back to the parent node\;
					$s_{\textit{leaf}}\leftarrow s_{\textit{leaf}-1}$\;
				}
				\Else{$a\leftarrow$D-UCB($s_{\textit{leaf}},\tau$)\Comment{eq. \eqref{eq:UCB_action}}\;
				$s_{\textit{leaf}}\leftarrow getNode(s_{\textit{leaf}},a)$\;}
			}
		}
		\KwRet $s_{\textit{leaf}}$\;
  }

	\Fn{simulation$(s_{t})$}{
  	$X_{T}^{i}\leftarrow$ trajectory from $s_{0}$ to $s_{t}$\;
  	\For{$j\leftarrow\ $depth$(s_{t})\ ${\normalfont to }$T$}    
    {
  		$X_{T}^{i}\leftarrow\{X_{T}^{i},\textit{randomWalk}(\mathbf{x}_{k+j}^{i})\}$\;
		}
		\KwRet $\mathcal{R}^{i}(X_{T}^{i})$\;
  }  
\end{algorithm}
\ULforem 

\subsubsection{Expansion} (line 5 of Algorithm \ref{al:tree}) 
The expansion phase expands the selected node with uniformly sampled action from the action space $\mathcal{A}(s)$ if the depth of the selected node does not exceed the search depth $T$. When the expanded node collides with an obstacle, the algorithm closes this edge ($C_{s}^{a}\leftarrow 1$) and returns to the selection phase.

\subsubsection{Simulation} (lines 9, 23-27 of Algorithm \ref{al:tree}) In the simulation phase, it calculates the informational reward $\mathcal{R}_{i}(X_T^i)$ of the selected trajectory. If the selected node's depth is less than the search depth $T$, it performs random walks. After that, the reward is calculated with the predicted trajectory $X_{T}^{i}$ as shown in Section \ref{subsec:informational_reward}.

\subsubsection{Backpropagation} (line 10 of Algorithm \ref{al:tree}) The edge variables are updated in a backward pass. The visit counts are incremented, $N_{s}^{a}\leftarrow N_{s}^{a}\gamma^{\tau-\tau_{s}^{a}}+1$, and the total action value is updated, $W_{s}^{a}\leftarrow W_{s}^{a}\gamma^{\tau-\tau_{s}^{a}}+R_{i}$. As described in the \textit{selection} step, the discount factor $\gamma$ is applied to reduce the weight of the previous value.

\renewcommand{\algorithmiccomment}[1]{\(\hfill\hfill\hfill\hfill\hfill\hfill\hfill\hfill\hfill\hfill\hfill\hfill\hfill\hfill\hfill\hfill\triangleright\ \)#1}

\normalem 
\begin{algorithm}[t]
\caption{Distributed MCTS with GP for agent $i$}\label{al:distributedMCTS}
\DontPrintSemicolon
  \KwInput{$\mathbf{x}^{i}_{k},\alpha^{i}_{m-1},\beta^{i}_{m-1},X_{T}^{i-}$}
  \KwOutput{$X_{T}^{i}$}
  $y^{i}_{k}\leftarrow$\emph{getMeasurement}$(\mathbf{x}^{i}_{k})$\Comment{eq. \eqref{eq:measurement_model}}\;
  $(\alpha^{i}_{m},\beta^{i}_{m})\leftarrow$\emph{updateGP}$(\alpha^{i}_{m-1},\beta^{i}_{m-1},\mathbf{x}^{i}_{k},y^{i}_{k})$\Comment{eq. \eqref{eq:GP_alpha_beta}}\;
  $(\alpha_{m}^{i},\beta_{m}^{i})\leftarrow$\emph{GPconsensus}$(\alpha_{m}^{i},\beta_{m}^{i})$\Comment{eq. (18) in \cite{jang2020multi}}\;
  $T_{k}^{i}\leftarrow$\emph{initializeTree}$(\mathbf{x}_{k}^{i})$\;
	
	\For{$\tau\leftarrow 1\ ${\normalfont to }$N_{\text{search}}$}    
	{
		$(\hat{\alpha}_{m}^{i},\hat{\beta}_{m}^{i})\leftarrow$\emph{TrajectoryMerging}$(X_{T}^{i-})$\Comment{eq. \eqref{eq:GP_alpha_beta2}}\;
  	$T_{k}^{i}\leftarrow$ \emph{TreeSearch}$(T_{k}^{i},\hat{\alpha}_{m}^{i},\hat{\beta}_{m}^{i},\tau)$\Comment{Algorithm \ref{al:tree}}\;
  	$X_{T}^{i}\leftarrow$ \emph{getBestTraj}$(T_{k}^{i})$\;
  	$X_{T}^{i-}\leftarrow$ \emph{communicateTraj}$(X_{T}^{i})$\;
  }
	\KwRet $X_{T}^{i}$\;
\end{algorithm}
\ULforem 

\section{Simulation Result}\label{sec:simulation}
This section presents environmental learning simulations on the various situations. The first simulation is on a time-invariant synthetic environment, and the second is on a dynamic environment based on the real-world meteorological dataset. These environmental models are unknown a priori, and each robot obtains the sensory data from the current location. Furthermore, since the communication range is finite, some agents may not be able to communicate with each other.

\begin{figure*}[ht]
\begin{center}
\includegraphics[trim = 0mm 0mm 0mm 0mm, clip, width=0.9\textwidth]{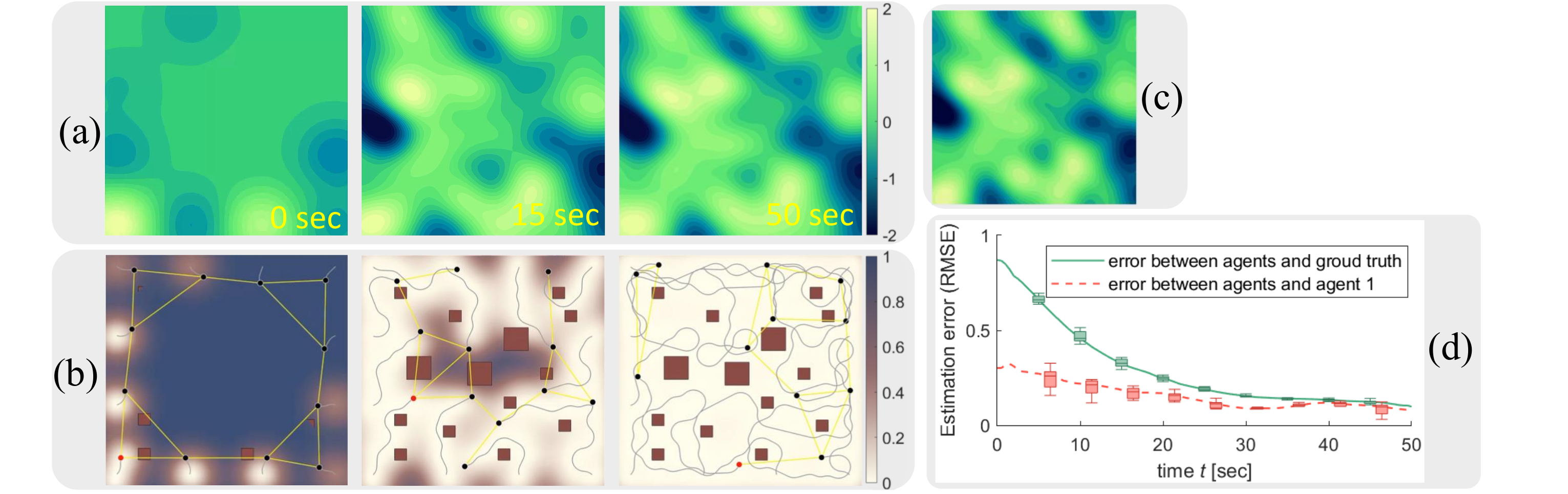}
\caption{Simulation 1-A. The process of the environmental model construction by 12 agents with fully distributed informative planning. (a) Change of GP estimate and (b) uncertainty propagation over time from $0$ to $50$ seconds in order from the left figure. (c) Ground truth of environmental model. Yellow lines in (b) indicate communication links between agents. All presented results are obtained by agent $\#1$ (red dot in (b)). (d) Environmental model estimation error. The solid green line shows the box plot of RMSE values between all agents and the ground truth. The dashed red line shows the box plot of RMSE values between all agents and the agent $\#1$. This graph means that all GP estimation results converge to the same result with only local communication.}
\label{fig:simulation1}
\end{center}
\vspace{-0.5cm}
\end{figure*}

\begin{figure*}[ht]
\centering
\begin{minipage}{0.3\textwidth}
{\includegraphics[width=1.0\textwidth]{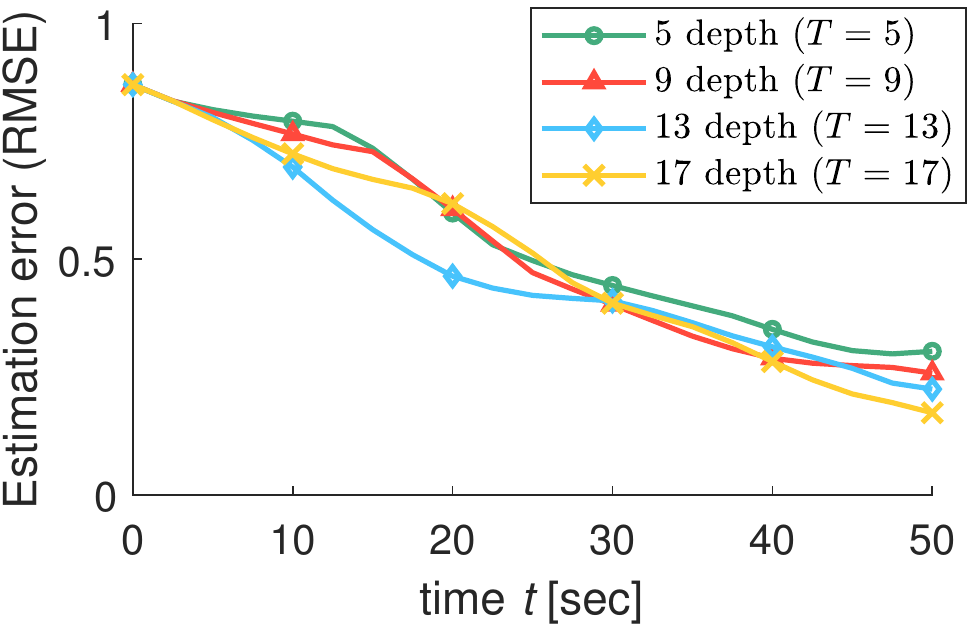}}
\caption*{(a) Search depths}
\end{minipage}
\begin{minipage}{0.3\textwidth}
\centering
{\includegraphics[width=1.0\textwidth]{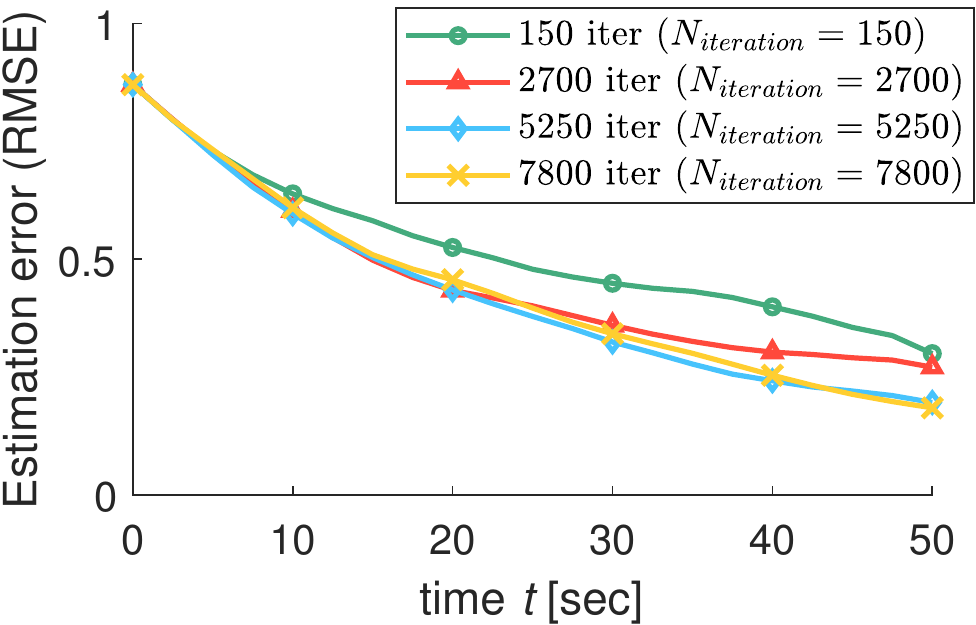}}
\caption*{(b) Number of iterations}
\end{minipage}
\begin{minipage}{0.3\textwidth}
\centering
{\includegraphics[width=1.0\textwidth]{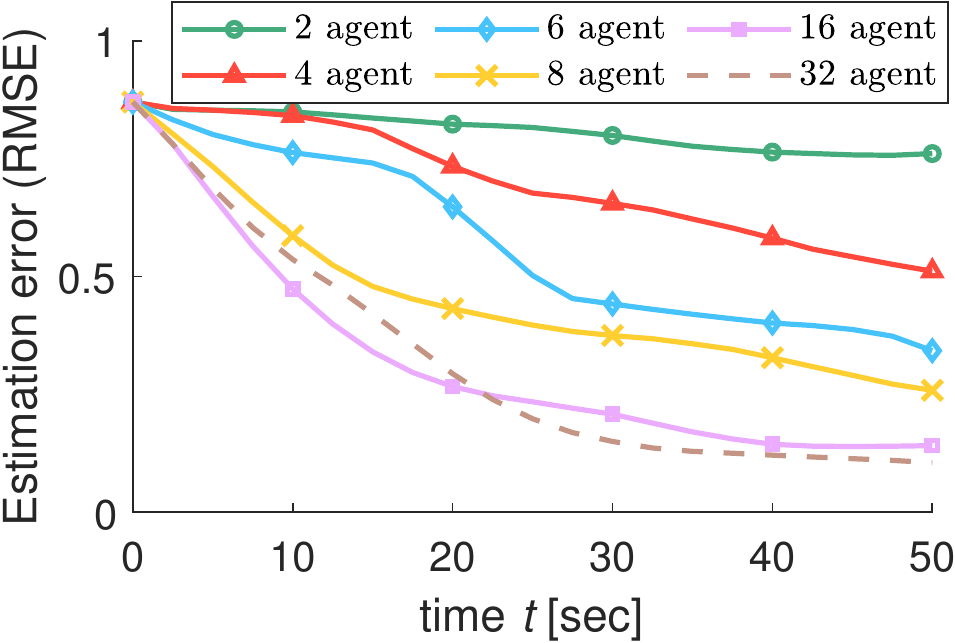}}
\caption*{(c) Number of agents}
\end{minipage}
\caption{Simulation 1-B. Environmental model construction with different conditions. (a) 6 agents' exploration results with different search depths. (b) 8 agents' exploration results with the different number of iterations. (c) Exploration results with the different number of agents.}
\label{fig:simulation2}
\vspace{-0.5cm}
\end{figure*}

\subsection{Simulation 1 - synthetic environment learning}
We perform the fully distributed informative planning simulation for multiple agents. They conduct exploration to obtain an estimate of the environmental map, considering collision avoidance and coordination. They can communicate only with neighbors within a range of 10 m (the map size is $20$ m $\times$ $20$ m) and move at $1$ m/s constantly. We set $\sigma_{s}^2=1$ and $\Sigma_l=\text{diag}([0.02, 0.02])$ for the Gaussian kernel \eqref{eq:kernel}, and we set $E=80$ for $E$-dimensional estimator \eqref{eq:GP_E_dim_mean_estimator} and \eqref{eq:GP_E_dim_var_estimator}.

The progress over time from $0$ to $50$ seconds is shown in Fig. \ref{fig:simulation1}. As shown in Figs. \ref{fig:simulation1}(a)-(b), twelve agents search the map together and generate the GP estimate presenting the ground truth model in \ref{fig:simulation1}(c). The agents scatter naturally and find the next locations to be updated based on the variance map. Also, as they avoid the places where the estimate is already reliable, they can minimize the redundant actions that can decrease the exploration efficiency. Through Fig. \ref{fig:simulation1}(b), it can be confirmed that the information of all agents is diffused through a communication link.

Although Figs. \ref{fig:simulation1}(a)-(b) show results from agent $\#1$ only, all the distributed GP estimates of each agent converge to the same by the average consensus as shown in Fig. \ref{fig:simulation1}(d). In other words, all agents do not simply use local information only in the exploration process but construct a global GP estimation map in a distributed manner.

\begin{figure}[ht]
\centering
\includegraphics[width=0.4\textwidth]{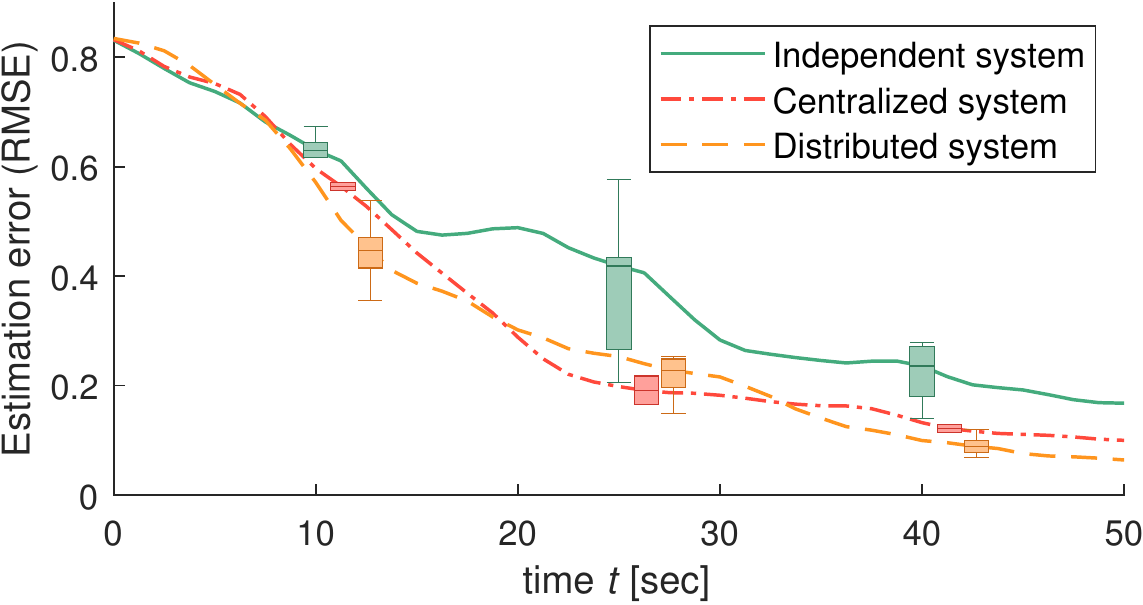} 
\caption{Simulation 1-C. 4 agents' exploration results with different type of systems for 10 trials (with different environments).} 
\label{fig:simulation3}
\vspace{-0.5cm}
\end{figure}

\begin{figure*}[ht]
\begin{center}
\includegraphics[trim = 0mm 0mm 0mm 0mm, clip, width=0.95\textwidth]{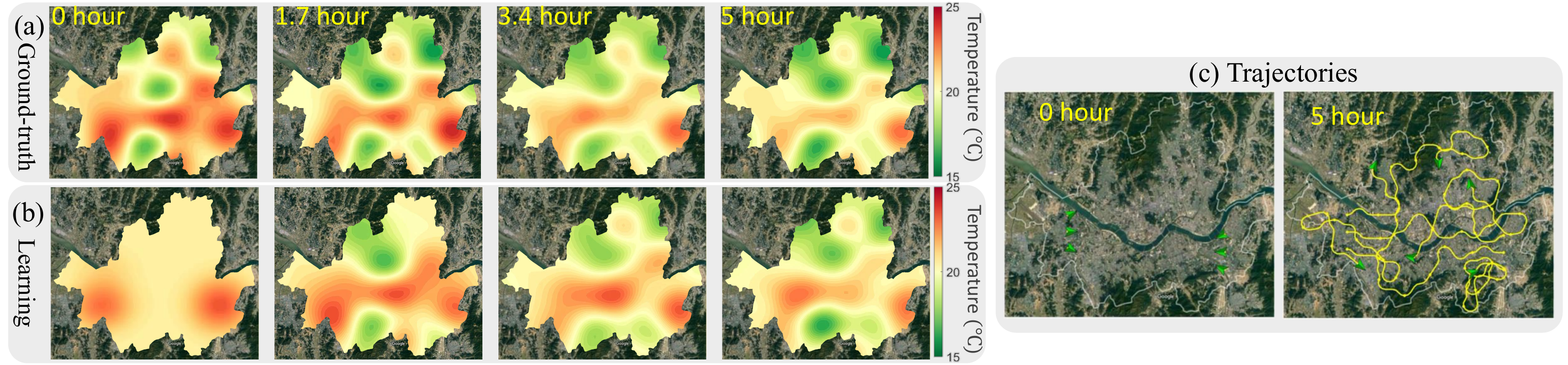}
\caption{Simulation 2. The process of the real-world heat map reconstruction performed by 6 moving UAVs with fully distributed informative planning. (a) Actual change of the heat map and (b) Environmental learning result from $0$ to $5$ hours in order from the left figure. (c) Initial locations and trajectories of UAVs.}
\label{fig:simulation4}
\end{center}
\vspace{-0.5cm}
\end{figure*}

We conduct more simulations in various environments to investigate the factors that affect search performance. In the tree search algorithm, the search depth $T$ and the number of iterations $N_{iterations}$ are the factors that directly affect the search result. The simulation results in Figs. \ref{fig:simulation2}(a)-(b) show that the search performance is proportional to both $T$ and $N_{iterations}$. The deeper the search, the more distant paths are considered. As the number of searching iterations increases, the probability of finding an optimal route increases. Fig. \ref{fig:simulation2}(c) shows that the search performance can be improved as the number of agents increases through a distributed algorithm. From these results, we can see that our algorithm is scalable for a large number of robots as well. 

Fig. \ref{fig:simulation3} compares the search performance of distributed systems, centralized systems, and independent systems. In the independent system, agents explore the area without communication. The centralized system has a central server that gathers all the information regardless of the communication range, and the server calculates paths for all agents.
Because the action space of centralized system $(\mathrm{n}(\mathcal{A})^n)$ is much bigger than that of distributed system $(\mathrm{n}(\mathcal{A}))$, we set about 22 times more $N_{iterations}$ for the centralized system than the distributed system. Even with limited communication and much fewer iterations, the distributed system performs similarly to the centralized system.

\subsection{Simulation 2 - real-world dataset environmental learning}
This section presents simulation result for the exploration in a dynamic environment (Fig. \ref{fig:overview}), using a real-world meteorological dataset. The simulation uses temperature data collected from weather stations in Seoul, South Korea \cite{ncdc}. The reason we choose the weather data for Seoul is that weather stations are densely distributed (the area of Seoul is $605.25\ \text{km}^{2}$). We create a heat map for a ground truth based on the data measured from 0 am to 5 am on July 16th, 2020.

In this scenario, a team of UAVs flies over the search area and gathers temperature data from the current UAV's location. They fly at a constant velocity of 20 km/h. Because their communication distance is limited to 20 km, sometimes they can be disconnected from one another. 

Some snapshots taken during the simulation are shown in Fig. \ref{fig:simulation4}. The ground truth over time is shown in Fig. \ref{fig:simulation4}(a), and the GP estimation is shown in Fig. \ref{fig:simulation4}(b). After 1.7 hours, the estimation result is similar to the ground truth. After that, the results track the true value continuously even when the actual environment changes.

\section{Conclusions}
This paper presents fully distributed robotic sensor networks to obtain a global environmental model estimate. We combine the Gaussian process with the Monte Carlo tree search in a distributed manner for peer-to-peer communication. Our method allows multiple robots to collaboratively perform exploration, taking into account collision avoidance and coordination. We validate our algorithm in various environments, including a time-varying temperature monitoring task using a real-world dataset. The results confirm that multiple agents can successfully explore the environment, and it is scalable with the increasing number of agents in the distributed network.


%


\ifCLASSOPTIONcaptionsoff
  \newpage
\fi



%

%
	
\vspace{-1cm}
\begin{IEEEbiography}[{\includegraphics[width=1in,height=1.25in,clip,keepaspectratio]{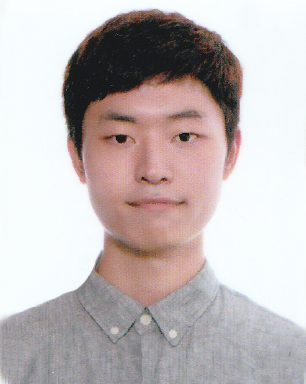}}]
{Dohyun Jang} received the B.S. degree in Electrical Engineering from Korea University in 2017, and the M.S. degree in Mechanical and Aerospace Engineering from Seoul National University, Seoul, in 2019. He is currently a Ph.D. Candidate in the School of Aerospace Engineering, SNU. His research interests include distributed systems, networked systems, machine learning, and robotics.
\end{IEEEbiography}

\vspace{-1cm}
\begin{IEEEbiography}[{\includegraphics[width=1in,height=1.25in,clip,keepaspectratio]{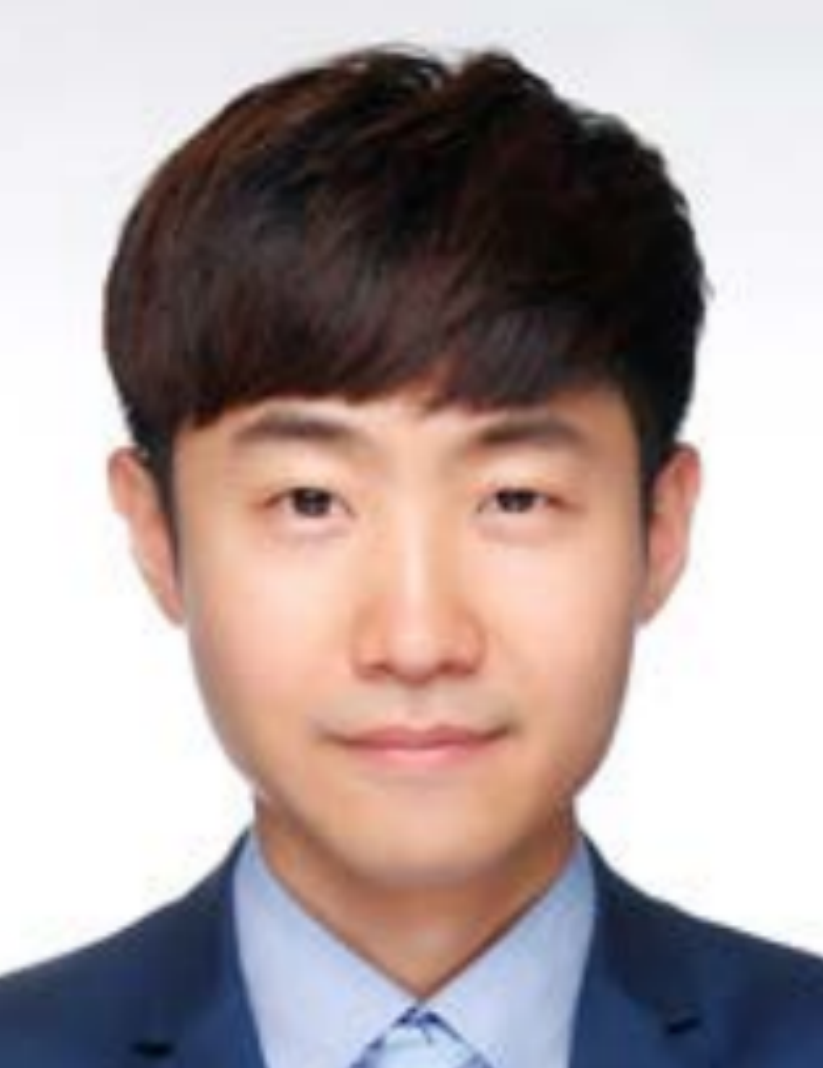}}]
{Jaehyun Yoo} received the Ph.D. degree in the School of Mechanical and Aerospace Engineering, Seoul National University, Seoul, in 2016. He was a postdoctoral researcher at the School of Electrical Engineering and Computer Science, KTH Royal Institute of Technology, Stockholm, Sweden. He is currently a Professor at the School of AI,
Sungshin Women’s University. His research interests include machine learning, indoor localization, automatic control, and robotic systems.
\vspace{8cm}
\end{IEEEbiography}

\vspace{-1cm}
\begin{IEEEbiography}[{\includegraphics[width=1in,height=1.25in,clip,keepaspectratio]{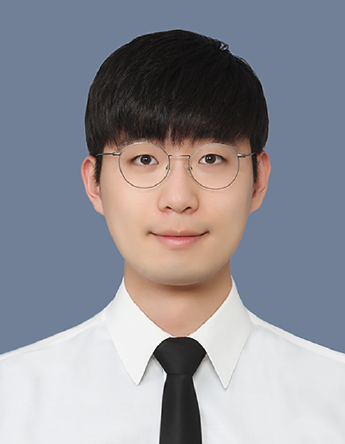}}]
{Clark Youngdong Son} Clark Youngdong Son received the B.S. degree in Mechanical Engineering from Sungkyunkwan University, and the Ph.D. degree in Mechanical and Aerospace Engineering from Seoul National University, in 2015 and 2021, respectively. He is currently a Staff Engineer at Mechatronics R\&D Center, Samsung Electronics. His research interests include robotics, path planning, and optimal control.
\end{IEEEbiography}

\vspace{-1cm}
\begin{IEEEbiography}[{\includegraphics[width=1in,height=1.25in,clip,keepaspectratio]{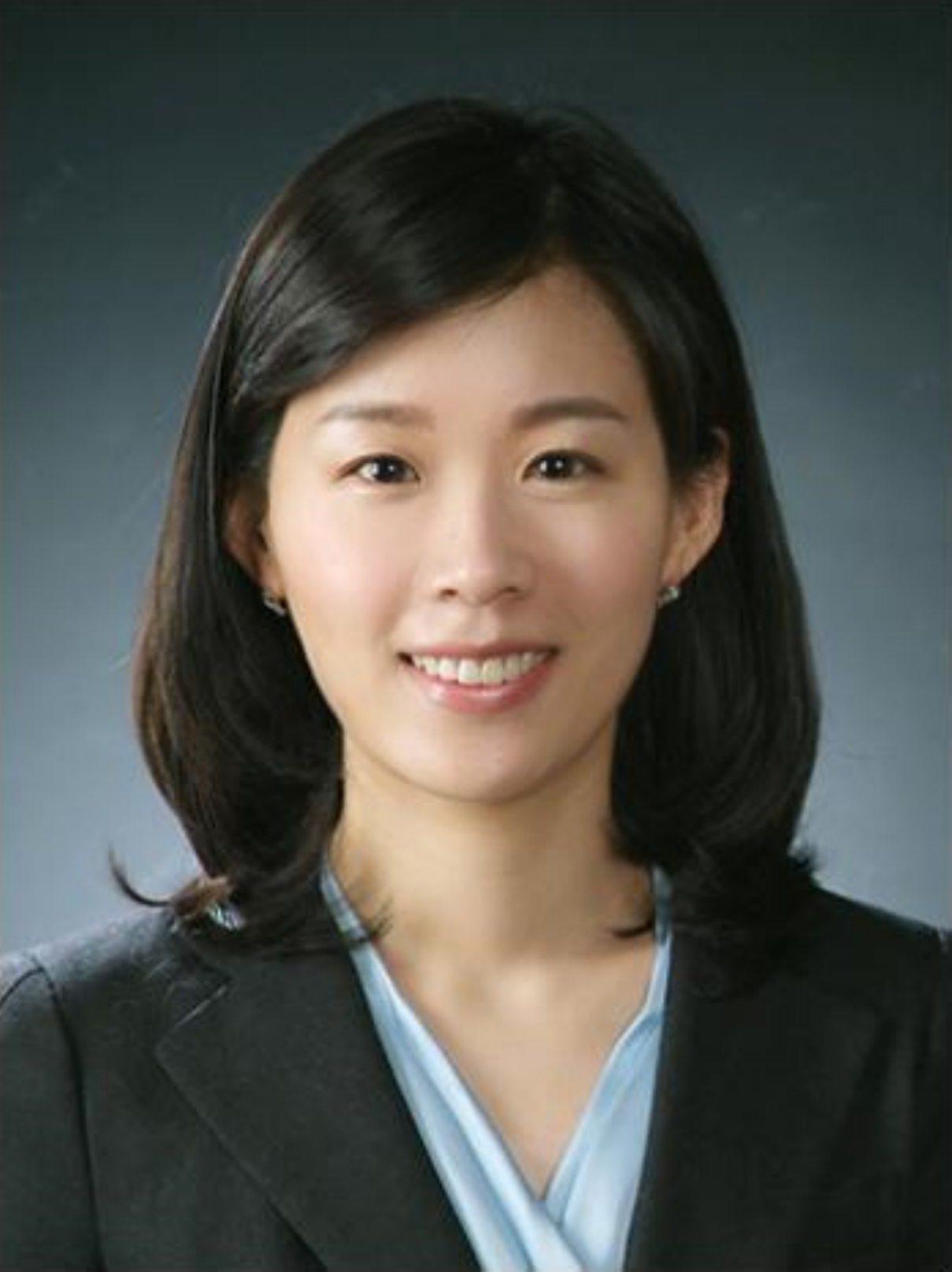}}]
{H. Jin Kim} received the B.S. degree from Korea Advanced Institute of Technology (KAIST) in 1995, and the M.S. and Ph.D. degrees in Mechanical Engineering from University of California, Berkeley, in 1999 and 2001, respectively. 

From 2002 to 2004, she was a Postdoctoral Researcher in Electrical Engineering and Computer Science, UC Berkeley. In 2004, she joined the Department of Mechanical and Aerospace Engineering at Seoul National University as an Assistant Professor, where she is currently a Professor. Her research interests include intelligent control of robotic systems and motion planning. 
\vspace{16cm}
\end{IEEEbiography}

\end{document}